\documentclass[twoside,11pt]{article}
\usepackage{jair, theapa, rawfonts}

\usepackage{times}  
\usepackage{helvet}  
\usepackage{courier}  
\usepackage{url}  
\usepackage{graphicx}  
\usepackage{amsthm}
\usepackage{thmtools}
\usepackage{thm-restate}
\usepackage{graphicx}
\usepackage{subfigure}

\usepackage{amsfonts}
\usepackage{amsmath}
\usepackage{bbm}
\usepackage{enumitem}
\usepackage{tabularx}
\usepackage{booktabs}
\usepackage{dcolumn}
\newcolumntype{d}[1]{D{.}{.}{#1} }

\usepackage{tikz}
\usepackage{pgfplots}	
\usetikzlibrary{arrows}
\usetikzlibrary{backgrounds,automata}
\usetikzlibrary{decorations.pathreplacing}

\usepackage{color}
\usepackage{colortbl}
\usepackage{multirow}

\definecolor{green}{RGB}{4,101,53}
\definecolor{red}{RGB}{160,30,50}
\newcommand{\clr}[2][]{   
       \cellcolor{red!#1} #2}
\newcommand{\clg}[2][]{   
       \cellcolor{green!#1} #2}

\renewcommand{\cal}[1]{\mathcal{#1}}
\newcommand{\tup}[1]{\langle #1 \rangle}
\newcommand{\R}{\mathbb{R}}
\newcommand{\D}{\mathcal{X}}
\newcommand{\E}{\mathbb{E}}
\newcommand{\X}{\D}

\newcommand{\indicator}{\mathbbm{1}}

\newcommand{\symdiff}{\triangle}
\newcommand{\Parzen}{\text{Parzen}}
\newcommand{\dist}[1]{\lVert #1 \rVert}
\newcommand{\QII}{\mathit{QII}}

\DeclareMathOperator*{\argmin}{arg\,min}
\DeclareMathOperator*{\argmax}{arg\,max}

\newcount\Comments  
\newcommand{\mytodo}[2]
{\ifnum\Comments=1
	{\marginpar{\tiny{\color{#1}	#2}}}\fi}

\newcommand{\kuba}[1]{\mytodo{blue}{[JS: #1]}}

\newcommand{\citename}[1]{\citeauthor{#1} (\citeyear{#1})}
\theoremstyle{plain}

\newtheorem{theorem}{Theorem}[section]

\newtheorem{lemma}[theorem]{Lemma}

\theoremstyle{definition}

\theoremstyle{remark}

\graphicspath{ {pictures/} }

\title{Axiomatic Characterization of Data-Driven Influence Measures for Classification}

\author{\name Jakub Sliwinski \email jsliwinski@ethz.ch \\
       \addr ETH Zurich\\
       Zurich, Switzerland
       \AND
       \name Martin Strobel \email mstrobel@comp.nus.edu.sg \\
       \addr National University of Singapore\\
       Singapore
       \AND
       \name Yair Zick \email zick@comp.nus.edu.sg \\
       \addr National University of Singapore\\
       Singapore}

\begin{document}

\maketitle

\begin{abstract}
We study the following problem: given a labeled dataset and a specific datapoint $\vec x$, how did the $i$-th feature influence the classification for $\vec x$?
We identify a family of {\em numerical influence measures} --- functions that, given a datapoint $\vec x$, assign a numeric value $\phi_i(\vec x)$ to every feature $i$, corresponding to how altering $i$'s value would influence the outcome for $\vec x$. 
This family, which we term {\em monotone influence measures (MIM)}, is uniquely derived from a set of desirable properties, or axioms.
The MIM family constitutes a provably sound methodology for measuring feature influence in classification domains; the values generated by MIM are based on the dataset alone, and do not make any queries to the classifier. While this requirement naturally limits the scope of our framework, we demonstrate its effectiveness on data. 
\end{abstract}

\section{Introduction}\label{sec:intro}
Recent years have seen the widespread implementation of data-driven decision making algorithms in increasingly high-stakes domains, such as finance, healthcare, transportation and public safety. Using novel ML techniques, these algorithms are able to process massive amounts of data and make highly accurate predictions; however, their inherent complexity makes it increasingly difficult for humans to understand {\em why} certain decisions were made. 
Indeed, these algorithms are {\em black-box decision makers}: their inner workings are either hidden from human scrutiny by proprietary law, or (as is often the case) are so complicated that even their own designers are hard-pressed to explain their behavior. By obfuscating their reasoning, data-driven classifiers expose human stakeholders to risks. These may include incorrect decisions (e.g. a loan application that was wrongly rejected due to system error), information leaks (e.g. an algorithm inadvertently uses information it should not have used), or discrimination (e.g. biased decisions against certain ethnic or gender groups). Government bodies and regulatory authorities have recently begun calling for {\em algorithmic transparency}: providing human-interpretable explanations of the underlying reasoning behind large-scale decision making algorithms. Several recent works propose making algorithms more transparent by using numerical influence measures: methods for measuring the importance of every feature in a dataset. However, these works, by and large, do not justify {\em why} their particular methodology is sound. 
Our work takes an axiomatic approach to influence measurement in data-driven domains. Starting from a set of desiderata (or {\em axioms}), we uniquely derive a class of measures satisfying these axioms. Thus, our work provides a
\begin{quote}
	{\em ...formal axiomatic analysis of automatically generated numerical explanations for black-box classifiers.}
\end{quote}

\subsection{Our Contribution}\label{sec:contrib}
{\em Numerical influence measures} are functions that assign a value $\phi_i$ to every feature $i$; $\phi_i$ corresponds to the predicted effect of this feature on the outcome. We identify specific properties (axioms) that any reasonable influence measure should satisfy (Section~\ref{sec:axioms}), and derive a class of influence measures, dubbed {\em monotone influence measures} (MIM), uniquely satisfying these axioms (Section~\ref{sec:char}). 
Next, we show that MIM can be interpreted as the solution to a natural optimization problem, further grounding our methodology (Section~\ref{sec:optimization}). 
Unlike most existing influence measures in the literature, we assume neither knowledge of the underlying decision making algorithm, nor of its behavior on points outside the dataset. Indeed, some methodologies (see Section~\ref{sec:existing} in the supplementary material) are heavily reliant on having access to counterfactual information: what would the classifier have done if some features were changed? This may be a strong assumption: it requires not only access to the classifier, but also the potential ability to use it on nonsensical data points\footnote{For example, if the dataset consists of medical records of men and women, the classifier might need to answer how it would handle pregnant men.}. By making no such assumptions, we provide a general methodology for measuring influence; indeed, many of the methods described in Section~\ref{sec:related} are unusable in the absence of classifier access, or when the underlying classification algorithm is not known. 
We show that despite our rather limiting conceptual framework, MIM does surprisingly well on a sparse image dataset, and provides an interesting analysis of a recidivism dataset. We compare the outputs of MIM to other measures, and provide interpretable results. Additional results relate MIM, new influence measures in a statistical cost sharing domain~\cite{balkanski2017cost}, and classic game-theoretic measures~\cite{banzhaf}.

\subsection{Related Work}\label{sec:related}
Algorithmic transparency has been debated and called for by government bodies \cite{goodman2017explanation,WhiteHouseReport2016BigData,Smith2016WhiteHouseReport}, the legal community \cite{Suzor2015DuffyvGoogle,Charruault2013CourDecassation}, and the media \cite{hofman2017prediction,Angwin2016Accountable,mittelstadt2016ethics,Winerip2016Parole}. The AI/ML research community is paying attention: algorithmic fairness, accountability and transparency is quickly gaining traction in the CS community, with new conferences (e.g. FAT* and AIES), numerous workshops and dozens of publications in mainstream AI/ML conferences. Several ongoing research efforts are informing the design of explainable AI systems (e.g. \citename{kroll2017accountable}, \citename{zeng2017interpretable}), as well as tools that explain the behavior of existing black-box systems (see \citename{weller2017transparency} for an overview); we focus on the latter.  

The work most closely related to ours is that of \citename{Datta2015influence}. \citename{Datta2015influence} axiomatically characterize an influence measure for datasets; however, they interpret influence as a global measure (e.g., what is the overall importance of gender for decision making), whereas we measure feature importance for individual datapoints. Moreover, as \citename{Datta2016algorithmic} show, the measure proposed by \citename{Datta2015influence} outputs undesirable values (e.g. zero influence) on real data; this is due to the fact that the measure requires the existence of counterfactual data: datapoints that differ by only a single feature. As we show in Section~\ref{sec:experiments}, MIM does not require such a dense dataset in order to register influence. \citename{baehrens2010explain} propose a data-driven influence measure that relies on a potential-like approach; as we demonstrate, their methodology fails to satisfy reasonable properties even on simple datasets.

Other approaches in the literature rely on black-box access to the classifier. \citename{Datta2016algorithmic} use an axiomatically justified influence measure based on an economic fairness paradigm, called QII; briefly, QII perturbs feature values and observes the effect this has on the classification outcome. Another line of work using black-box access \cite{Ribeiro2016should,Ribeiro2016model} uses queries to the classifier in a local region near the point of interest in order to measure influence. \citename{adler2016auditing} equate the influence of a given feature $i$ with the ability to infer $i$'s value from the rest of features, after it has been obscured; this idea is the basis for a framework for auditing black-box models. However, this approach assumes that one can make predictions on a dataset with some features removed. \citename{koh2017understanding} have a different take on influence, identifying key {\em datapoints} --- rather than features --- that explain classifier behavior.

Some works study explanations for specific domains, such as neural networks \cite{ancona2017dnns,shrikumar2017deeplift,sundararajan2017axiomatic}, or computer programs~\cite{datta2017programs}; others apply explanations for generating more accurate predictions~\cite{ross2017rightforrightreasons}.

\section{The Model}\label{sec:prelim}
A dataset $\cal X= \tup{\vec x_1,\dots,\vec x_m}$ is given as a list of vectors in $\R^n$ (each dimension $i \in [n]$ is a feature), where every $\vec x_j \in \cal X$ has a unique label $c_j\in \{-1,1\}$; given a vector $\vec x \in \cal X$, we refer to the label of $\vec x$ as $c(\vec x)$. 
An {\em influence measure} is a function $\phi$ whose input is a dataset $\cal X$, vector labels denoted by $c$, and a specific {\em point of interest} $\vec x\in \cal X$; its output is a value $\phi_i(\vec x,\cal X,c)\in \R$; we often omit the inputs $\cal X$ and $c$ when they are clear from context. The value $\phi_i(\vec x)$ should correspond to how altering the $i$-th feature is predicted to affect the outcome $c(\vec x)$ for $\vec x$ in the following way: if $\phi_i(\vec x)$ is positive (negative), then for points similar to $\vec x$, increasing the value of the i-th feature increases (decreases) the likelihood of assigning the label $c(\vec x)$, and the value $|\phi_i(\vec x)|$ expresses the strength of that effect.

\section{Axioms for Data-Driven Influence}\label{sec:axioms}
We are now ready to define our axioms; these are simple properties that we believe any reasonable influence measure should satisfy. 
\begin{enumerate}[style=unboxed,leftmargin=0cm]
	\item {\bf Shift Invariance:} let $\cal X + \vec b$ be the dataset resulting from adding the vector $\vec b \in \R^n$ to every vector in $\cal X$ (not changing the labels). An influence measure $\phi$ is said to be {\em shift invariant} if for any vector $\vec b \in \R^n$, any $i \in [n]$ and any $\vec x \in \cal X$, 
	$$\phi_i(\vec x,\cal X) = \phi_i(\vec x+\vec b,\cal X+\vec b).$$
	In other words, shifting the entire dataset by some vector $\vec b$ should not affect feature importance.\label{ax:shift}
	\item {\bf Rotation and Reflection Faithfulness:} let $A$ be a rotation (or reflection) matrix, i.e. an $n\times n$ matrix with $\det(A) \in \pm 1$; let $A\cal X$ be the dataset resulting from taking every point $\vec x$ in $\cal X$ and replacing it with $A\vec x$. An influence measure $\phi$ is said to be {\em rotation and reflection faithful} if for any rotation matrix $A$, and any point $\vec x \in \cal X$, we have 
	$$A\phi(\vec x,\cal X) = \phi(A\vec x,A\cal X).$$
	In other words, the influence measure $\phi$ is invariant under rotation and reflection. \label{ax:rotRef}
	\item {\bf Continuity:} an influence measure $\phi$ is said to be {\em continuous} if it is a continuous function of $\cal X$.\label{ax:cont}
	\item {\bf Flip Invariance:} let $-c$ be the labeling resulting from replacing every label $c(\vec x)$ with $-c(\vec x)$. An influence measure is {\em flip invariant} if for every point $\vec x \in \cal X$ and every $i \in [n]$ we have 
	$\phi_i(\vec x,\cal X,c) = \phi_i(\vec x,\cal X,-c).$\label{ax:flip}
	\item {\bf Monotonicity:} a point $\vec y\in \R^n$ is said to {\em strengthen} the influence of feature $i$ with respect to $\vec x \in \cal X$ if $c(\vec x) = c(\vec y)$ and $y_i > x_i$; similarly, a point $\vec y \in \R^n$ is said to {\em weaken} the influence of $i$ with respect to $\vec x\in \cal X$ if $y_i > x_i$ and $c(\vec x) \ne c(\vec y)$. An influence measure $\phi$ is said to be \emph{monotonic}, if for any data set $\cal X$, any feature $i$ and any data point $\vec x \in \cal X$ we have $\phi_i(\vec x,\cal X) \le \phi_i(\vec x,\cal X\cup\{\vec y\})$ if $\vec y$ strengthens $i$ w.r.t. $\vec x$, and $\phi_i(\vec x,\cal X) \ge \phi_i(\vec x,\cal X\cup\{\vec y\})$ if $\vec y$ weakens $i$ w.r.t. $\vec x$. \label{ax:mono}
	\item {\bf Non-Bias:}
	suppose that all labels for points in $\cal X$ are assigned i.i.d. uniformly at random (i.e. for all $\vec y \in \cal X$, $\Pr[c(\vec y) = 1]= \Pr[c(\vec y) = -1]$). We call this label distribution $\cal U$; an influence measure $\phi$ satisfies the {\em non-bias} axiom if for all $\vec x \in \cal X$ and all $i\in [n]$ we have 
	\begin{align*}
		\E_{c\sim \cal U}[\phi_i(\vec x,\cal X,c)\mid c(\vec x)] = 0
	\end{align*}
	In other words, when we fix the label of $\vec x$ and randomize all other labels, the expected influence of all features is 0. \label{ax:zeroLabel}
\end{enumerate}
The first four axioms are rather fundamental: indeed, most influence measures in the literature trivially satisfy some variants of these properties. The last two axioms are more interesting. While we strongly believe that there is no one ``universally correct'' set of axioms that all influence measures should satisfy, we believe that our proposed properties make intuitive sense in many application domains.
\subsection{The Case for Monotonicity} \label{sec:whymono}
Monotonicity is a key defining property for characterizing our family of influence measures. Intuitively, it is a consistency requirement: if one is to argue that a person's old age caused their bank loan to be rejected, then finding {\em older} persons whose loans were similarly rejected should strengthen this argument; however, finding older persons whose loans were not rejected should weaken the argument. We mention that monotonicity coupled with flip invariance implies the converse argument as well: adding younger persons whose loans were accepted should increase the influence of age, and adding younger persons whose loans were rejected would decrease it. Of course, in order for the monotonicity property to make any sense, feature states must satisfy some natural order: they should be numerical quantities (e.g. income, age, scores in a test, or shades of a color), states with a natural progression (e.g. education level, or disease severity), or binary states (e.g. gender). Monotonicity does not easily apply to features whose states cannot be naturally ordered (e.g. profession, ethnicity, species). That said, our characterization result holds whenever the dataset has at least one feature whose states can be naturally ordered.
\subsection{Non-Bias: Measuring Influence vs. Measuring Noise}\label{sec:whyrandomlabels}
The Non-Bias axiom states that when labels are randomly generated, no feature should have any influence in expectation. We argue that this requirement is absolutely necessary: any influence measure that fails this test exhibits an inherent bias towards some features, even when labels are completely unrelated to the data. As we show in Section~\ref{sec:existing} in the supplementary material, some measures in the literature fail the non-bias test.

\section{Characterizing Monotone Influence Measures}\label{sec:char}
In what follows, we show that influence measures satisfying the axioms in Section~\ref{sec:axioms} must follow a specific formula, described in Theorem~\ref{thm:charthm}. Below, $\indicator(p)$ is a $\{1,-1\}$-valued indicator (i.e. $1$ if $p$ is true and $-1$ otherwise), and $\dist {\vec x}$ is the Euclidean length of $\vec x$; our analysis admits other distances over $\R^n$, but we stick with $\dist{\cdot}$ for concreteness.
We begin by showing a simple technical lemma.
\begin{lemma}\label{monolemma}
	If an influence measure $\phi$ satisfies both monotonicity and rotation faithfulness, then for any dataset $\cal X$, any datapoint $\vec x \in \X$, and any $\vec y$ where $\vec y$ and $\vec x$ differ in some feature, there exists some $a \in \R$ such that
	\begin{equation} \label{eqmono}
	\phi(\vec x, \cal X \cup \{\vec y\}) - \phi(\vec x, \cal X) = a(\vec y - \vec x);
	\end{equation} 
	furthermore, $a\ge 0$ if $c(x) = c(y)$, and $a \le 0$ otherwise.
\end{lemma}
\begin{proof}
	Suppose for contradiction that there are $ \cal X , \vec x \in \cal X$, and $\vec y \in \R^n $ with $c(\vec x) = c(\vec y)$ such that
	\begin{align*}
	\forall a \in \R_+\colon \underbrace{\phi(\vec x,\D) - \phi(\vec x,\D \cup \{ \vec y \})}_{:= \vec l } \neq a  (\vec x - \vec y )
	\end{align*}
Let $A$ be rotation matrix such that $(A\vec l)_1 <0$ and $A (\vec x - \vec y )_1 > 0 $; such a matrix exists since either the two vectors are linearly independent, or $\vec l = -b (\vec x - \vec y )$ for some $b \in R_+$. Since $\phi$ satisfies Axiom~\ref{ax:rotRef} (Rotation), we get 
	$$\phi_1(\vec Ax,A\D) - \phi_1(A\vec x,A\D \cup \{ A\vec y \}) < 0,$$ 
	which contradicts the first case of Axiom~\ref{ax:mono} (Monotonicity). The case where $c(\vec x) \neq c(\vec y)$ can be derived symmetrically.    
\end{proof}
We are now ready to prove our main result. 
\begin{theorem}\label{thm:charthm}
An influence measure $\phi$ satisfies axioms \ref{ax:shift} to \ref{ax:zeroLabel} 
iff it is of the form
\begin{equation}
\phi(\vec x,\cal X,c) = \sum_{\vec y \in \cal X  \setminus \vec x} (\vec y- \vec x)\alpha(\dist{\vec y - \vec x}) \indicator(c(\vec x) = c(\vec y))\label{eq:MIM}
\end{equation}
where $\alpha$ is any non-negative-valued function. 
\end{theorem}
\begin{proof}
Suppose $\phi$ satisfies Axioms \ref{ax:shift} to \ref{ax:zeroLabel}. 
We prove the statement by induction on $ k =|\cal X|$.
First assume that $k = 1$. 
When $k = 1$, $\cal X = \tup{\vec x}$. By shift invariance, $\phi(\vec x, \X) = \phi(\vec 0, \tup{\vec 0})$. The vector $\vec 0$ and $\tup{\vec 0}$ are invariant under rotation; hence, by rotation faithfulness, $\phi(\vec 0, \tup{\vec 0}) = \vec 0$, the only vector invariant under rotation. In other words, whenever the dataset has a single point, all features have zero influence.  

When $k=2$, we have $\X = \tup{\vec x, \vec y}$. 
If $\vec x=\vec y$ all features have zero influence (this is irrespective of whether $c(\vec x) =c(\vec y)$ or $c(\vec x) \ne c(\vec y)$).
Further, note that any set of two points can be translated by shift and rotation to any other set of two points with the same labels and the same euclidean distance between them. Hence, by shift invariance, rotation faithfulness and Lemma~\ref{monolemma},
$$\phi(\vec x) = 
\begin{cases}
    (\vec y - \vec x)\alpha_1(\dist{\vec y - \vec x})       & \quad \text{if } c(\vec x) = c(\vec y)\\
   (\vec y - \vec x)\alpha_2(\dist{\vec y - \vec x}) & \quad \text{if } c(\vec x) \neq c(\vec y),
  \end{cases}
$$
where $\alpha_1$ ($\alpha_2$) is some non-negative (non-positive) valued function that depends only on $\dist{\vec y - \vec x}$. By random labels and flip faithfulness, $\alpha_1 = - \alpha_2$, thus
$\phi(\vec x, \X) = (\vec y - \vec x)\alpha(\dist{\vec y - \vec x}) \indicator(c(\vec x) = c(\vec y))$, where $\alpha$ depends only on $\dist{\vec y - \vec x}$.

Suppose the hypothesis holds when $|\cal X|\le k$. Consider any dataset $\cal Y$ of size $k+1$. The cases where the dataset $\cal Y$ does not contain at least three different points are handled in a manner similar to when $k = 1,2$.  
\kuba{To first talk about three different points and then two distinct points, is so confusing that I'm not sure if it's correct}
Suppose $\cal Y$ contains at least two distinct datapoints $\vec y,\vec z \neq \vec x$. We prove the hypothesis for the case where $\vec y - \vec x$ and $\vec z - \vec x$ are linearly independent; the case where they are linearly dependent follows from continuity (we can `perturb' the points slightly to avoid linear dependency).
By Lemma~\ref{monolemma}
\begin{align*}
\phi(\vec x, \cal Y) \in A =& \{\phi(\vec x, \cal Y \setminus \{\vec y\}) + a(\vec y - \vec x) : a \in \R\}\\
\text{ and }
\phi(\vec x, \cal Y) \in B =& \{\phi(\vec x, \cal Y \setminus \{\vec z\}) + a(\vec z - \vec x) : a \in \R\}.
\end{align*}
Further by the inductive hypothesis, $\phi(\vec x, \cal Y \setminus \{\vec y\})$ equals
\begin{align*}
\phi(\vec x, \cal Y \setminus \{\vec y, \vec z\})  +(\vec z - \vec x)\alpha(\dist{\vec z - \vec x})\indicator(c(\vec x) = c(\vec z))
\end{align*}
and $\phi(\vec x, Y \setminus \{\vec z\})$ equals
\begin{align*}
\phi(\vec x, Y \setminus \{\vec y, \vec z\})+(\vec y - \vec x)\alpha(\dist{\vec y - \vec x})\indicator(c(\vec x) =c(\vec y)).
\end{align*}
Since $\vec y - \vec x$ and $\vec z - \vec x$ are linearly independent we get,
\begin{align*}
\phi(\vec x, \cal Y) \in  A \cap B &= \{\phi(\vec x, \cal Y \setminus \{\vec y, \vec z\})\\
&+(\vec z - \vec x)\alpha(\dist{\vec z - \vec x})\indicator(c(\vec x)=c(\vec z)) \\ &+ (\vec y - \vec x)\alpha(\dist{\vec y - \vec x})\indicator(c(\vec x)=c(\vec y))\}
\end{align*}
concluding the inductive step.
\end{proof}
We refer to measures satisfying Equation~\eqref{eq:MIM} as {\em monotone influence measures} (MIM). We note that MIM is a {\em family} of influence measures, parameterized by the choice of the function $\alpha$. It may be natural to assume that $\alpha$ is a monotone decreasing function; that is, the further away the point $\vec y$ is from $\vec x$, the lower its effect on $\phi$ should be. However, this assumption does not follow from our analysis.
In what follows, we propose a method of selecting the $\alpha$ parameter, by viewing MIM as a solution to an optimization problem, in a similar manner to \citename{Ribeiro2016should}.

\section{Choosing Optimal MIM Parameters}\label{sec:optimization}
Is MIM a `good' way of measuring influence? If the reader is convinced that the axioms proposed in Section~\ref{sec:axioms} make sense, then our work here is done. In this section we make an additional case for MIM, showing that it is an optimal solution to a natural optimization problem. The results in this section serve an additional important purpose. Our characterization result (Theorem~\ref{thm:charthm}) identifies a {\em family of measures} (MIM), not a unique function, parameterized by the $\alpha$ function in Equation~\eqref{eq:MIM}. Theorem~\ref{thm:charthm} only requires that $\alpha$ is a function of $\dist{\vec x - \vec y}$, but does not indicate what choice of $\alpha$ is appropriate. As we now show, MIM can be seen as a solution to an underlying optimization problem, the parameters of which may indicate the appropriate choice of $\alpha$.

We are given a dataset $\cal X$ and a point of interest $\vec x$. Consider any potential influence vector $\phi$; intuitively, $\phi$ should be a direction, such that moving $\vec x$ along $\phi$ will `increase the chance' or `positively contribute' to the label of $\vec x$ being $c(\vec x)$. For any point $\vec y \in \cal X$ s.t. $c(\vec y) = c(\vec x)$, it is desirable that $\phi$ points towards $\vec y$; if $c(\vec y) \neq c(\vec x)$, $\phi$ should point away from $\vec y$. 

Local points should be assigned more influence than further ones. Assume a function $\alpha_0: \R \to \R$ whose input is $\dist{(\vec y - \vec x)}$; its output is a weightage representing the importance of $\vec y$; intuitively, $\alpha_0$ should be monotone decreasing in its input, assigning higher values to points in a local neighborhood of $\vec x$ and lower importance to points further away. Hence, $\phi(\vec x, \cal X)$ should maximize
\begin{align}
\sum_{\vec y \in \cal X} \alpha_0(\dist{\vec y - \vec x})\cos(\vec y - \vec x, \phi)\indicator(c(\vec x) = c(\vec y))
\label{eq:MIM-opt}
\end{align}
Equation \eqref{eq:MIM-opt} can be thought of as a weighted variant of the total cosine similarity optimization target.
\begin{theorem}
	MIM with the $\alpha$ parameter in \eqref{eq:MIM} set to $\alpha(\dist{\vec y - \vec x}) = \frac{\alpha_0(\dist{\vec y - \vec x})}{\dist{\vec y - \vec x}}$, maximizes \eqref{eq:MIM-opt}.
\end{theorem}
\begin{proof}
For ease of exposition we assume that $\cal X$ has been shifted so that $\vec x = \vec 0$; since MIM is shift invariant this is no loss of generality.
Note that \eqref{eq:MIM-opt} treats a point $\vec y$ with $c(\vec y) \neq c(\vec x)$ as it would the point $\vec z = - \vec y$ with $c(\vec z) = c(\vec x)$. To simplify further, we assume that all points with a different label than $\vec x$ are swapped for their negatives with the same label as $\vec x$, resulting in a simplified formula
\begin{equation}\label{eq1}
\phi(\vec x, \cal X) := \argmax_{\phi \in \R^n} \sum_{\vec y \in \cal X} \cos(\vec y, \phi)\alpha_0(\dist{\vec y}).
\end{equation}
Equation \eqref{eq1} pertains to the direction of $\phi$. Intuitively, the length $\dist{\phi}$ should correspond to how well the problem can be optimized. If the dataset is random, no direction should be particularly good, resulting in a short $\phi$; that is, small influence values for every feature. In the case of the opposite extreme --- all points with the same label as $\vec x$ are in a similar region, and all points with a different label in another --- $\phi$ should be long, indicating high influence towards the points with the same label. Hence, the most natural way to specify the length of $\phi$, again assuming $\vec x = \vec 0$ for simplicity, is:
\begin{equation}
\dist{\phi(\vec x, \cal X)} := \max_{\phi \in \R^n} \sum_{\vec y \in \D} \cos(\vec y, \phi)\alpha_0(\dist{\vec y}).\label{eq2}
\end{equation}

\noindent To show that MIM maximizes \eqref{eq:MIM-opt}, we require the following lemma.
\begin{lemma}\label{lem:cosine}
Let $f : \R^n \to \R$. Given a dataset $\D$, 
$$
\sum_{\vec y \in \D} \cos(\vec y, \phi)f(\vec y)=\dist{\sum_{\vec y \in \D} \frac{\vec y\cdot f(\vec y)}{\dist{\vec y}} } \cos \Big( \phi, \sum_{\vec y \in \D} \frac{\vec y \cdot f(\vec y)}{\dist{\vec y}}  \Big).
$$
\end{lemma}
\begin{proof}
We show that this equation holds for two vectors; the general case easily follows.
$\cos(\vec y, \phi)f(\vec y) + \cos(\vec z, \phi)f(\vec z)$ equals 
\begin{align*}
\frac{\vec y \cdot \phi}{\dist{\vec y} \dist{\phi}} f(\vec y) + \frac{\vec z \cdot \phi}{\dist{\vec z} \dist{\phi}} f(\vec z) &=\\
\frac{\frac{\vec y}{\dist{\vec y}} f(\vec y) \cdot \phi + \frac{\vec z}{\dist{\vec z}} f(\vec z) \cdot \phi}
{\dist{\phi}} =\frac{\left( \frac{\vec y}{\dist{\vec y}} f(\vec y) + \frac{\vec z}{\dist{\vec z}} f(\vec z) \right) \cdot \phi}
{\dist{\phi}} &=\\
\dist{\frac{\vec y}{\dist{\vec y}} f(\vec y) + \frac{\vec z}{\dist{\vec z}} f(\vec z)} \cos \bigg( \phi, \frac{\vec y}{\dist{\vec y}} f(\vec y)  + \frac{\vec z}{\dist{\vec z}} f(\vec z) \bigg)
\end{align*}
\end{proof}
Using Lemma~\ref{lem:cosine} and substituting $\alpha_0(\vec x) = \dist{\vec x} \alpha(\dist{\vec x})$, the right-hand side of Equation (\ref{eq2}) becomes 
\[\max_{\phi \in \R^n} \Bigg( \dist{\sum_{\vec y \in \D} \vec y \alpha(\dist{\vec y})} \cdot \cos \big( \phi, \sum_{\vec y \in \D} \vec y \alpha(\dist{\vec y}) \big) \Bigg).\] 
Hence,
$ \dist{\phi(\vec x, \D)} = \dist{\sum_{\vec y \in \D} \vec y \alpha(\dist{\vec y})}. $
Combining that with Equation~\eqref{eq1} we get $\phi(\vec x, \cal X) := \sum_{\vec y \in \cal X\setminus\{\vec x\}} \vec y \alpha(\dist{\vec y})$.
Accounting for the simplifications we assumed, we get the general formula:
$$
\phi(\vec x, \cal X) := \sum_{\vec y \in \D} (\vec y - \vec x) \alpha(\dist{\vec y - \vec x}) \indicator(c(\vec x) = c(\vec y)).
$$
\end{proof}

Intuitively, given a point of interest $\vec x \in \cal X$, a monotone influence vector will point in the direction that has the `most' points in $\cal X$ that share a label with $\vec x$. The value $\dist{\phi}$ can be thought of as one's confidence in the direction: if $\dist{\phi}$ is high, this means that one is fairly certain where other vectors sharing a label with $\vec x$ are (and, correspondingly, this means that there are at least some highly influential features identified by $\phi$); in the case that $\dist{\phi}$ is small, the direction of $\phi$ is not a particularly strong indication of where other vectors of the same type can be found. In terms of choosing the right $\alpha$ parameter, Lemma~\ref{lem:cosine} provides a few useful insights: if we select $\alpha(\dist{\vec y -\vec x}) = 1$, then the resulting MIM measure maximizes the function
$\sum_{\vec y \in \cal X}\cos(\vec y- \vec x,\phi)\dist{\vec y - \vec x}$;
in other words, we put {\em more} weight on vectors in $\cal X$ that are more distant from $\vec x$. Similarly, if we choose $\alpha(\dist{\vec y - \vec x}) = \frac{1}{\dist{\vec y - \vec x}}$ then we place equal importance on all points in the dataset, whereas if we set $\alpha(\dist{\vec y - \vec x}) = \frac{1}{\dist{\vec y - \vec x}^2}$, vectors that are farther away from the point of interest are weighted by $\frac{1}{\dist{\vec y - \vec x}}$. This choice of $\alpha$ informs our implementation in Section~\ref{sec:experiments}.

\section{Influence: From Game-Theory to Data Analytics}\label{sec:gametheory}
Influence measurement is studied in various domains; while influence measures in data classification are a relatively recent research topic, influence has been studied extensively in {\em cooperative game theory}. In what follows, we apply our analysis to cooperative games. Beyond its mathematical interest, the purpose of this section is to show that our work is taking a step towards a {\em unified theory of influence measurement}, connecting game theory to a more general domain of influence measurement; indeed, ideas from cooperative games have been successfully applied in the machine learning domain in previous studies~\cite{Datta2015influence,Datta2016algorithmic,sundararajan2017axiomatic}.

A cooperative game is given by a set of {\em players} $N = \{1,\dots,n\}$ and a {\em characteristic function} $v:2^N \to \R$, assigning a value $v(S)$ to every subset of players (also referred to as a {\em coalition}). 
Translating to our setting, we can think of the players as features, and of sets as indicator vectors in $\{0,1\}^n$; 
thus, our dataset $\cal X$ consists of all vectors in $\{0,1\}^n$, where the label of the indicator vector corresponding to $S$ (which we denote $\vec e_S$) is the value $v(S)$.  
Note that for a fully faithful translation we'll need all sets to have binary values (i.e. $v(S) \in \{0,1\}$); 
cooperative games where all sets have values in $\{0,1\}$ are known as {\em simple games}\footnote{There are several excellent textbooks on cooperative game theory; we refer our reader to \cite{compcoopbook,coopbook}}; 
however, our definition easily extends to all types of cooperative games. What does MIM look like translated to this domain? Fixing a set $S \subseteq N$, taking the standard Hamming set distance with $\alpha(d) = \frac1d$, we obtain the following equation

\begin{align}
\phi(S) =& \sum_{T \subseteq N} \frac{v(S) - v(T)}{|S\symdiff T|}(\vec e_S - \vec e_T)\label{eq:influence-gt-basic}
\end{align}
In Equation~\eqref{eq:influence-gt-basic}, $S\symdiff T$ is the symmetric difference between $S$ and $T$; furthermore
\begin{align*}
(\vec e_S - \vec e_T)_i = &\begin{cases}1 & \mbox{if } i \in S\setminus T\\ -1& \mbox{if }i \in T\setminus S \\ 0. &\mbox{otherwise.}\end{cases}
\end{align*}
If $i \notin S$, $\phi_i(S)$ is of the form
\begin{align}
\phi_i(S) =& \sum_{T:i\in T} \frac{v(T) - v(S)}{|S\symdiff T|}\label{eq:influence-gt-i-notin-S}
\end{align}
One can think about Equation~\eqref{eq:influence-gt-i-notin-S} as a generalization of the key concept of {\em marginal contribution} in cooperative games. The marginal contribution of a player $i\in N$ to a set $S$ is the value $v(S\cup\{i\}) - v(S)$; this value is just one of the summands in Equation~\eqref{eq:influence-gt-i-notin-S}, when one takes $T = S\cup\{i\}$. Intuitively, one can think of $v(T) - v(S)$ as the benefit (or loss) resulting from forming the coalition $T$ (that contains $i$) over forming $S$; normalizing by $|S\symdiff T|$ ensures that player $i$ receives an equal share of the responsibility for this change: the players in $S\symdiff T$ need to take action in order for $T$ to form rather than $S$ (those in $T\setminus S$ need to join and those in $S\setminus T$ need to leave). Thus, Equation~\eqref{eq:influence-gt-i-notin-S} measures the overall capability of $i$ to affect a change in the value of $S$ by adding or removing players from $S$; the total added benefit of doing so is normalized by the extent of the change required. 

Let us take $S = \emptyset$; we obtain 
\begin{align}
\phi_i(\emptyset) = \sum_{T: i\in T}\frac{v(T)}{|T|}\label{eq:influence-stat-cost-sharing}
\end{align}
Equation~\eqref{eq:influence-stat-cost-sharing} is of particular interest: the same equation appears in an axiomatic characterization of influence in a data-dependent cooperative setting given by \citename{balkanski2017cost}. In their work, \citeauthor{balkanski2017cost} explore player influence where the dataset $\cal X$ consists of $m$ coalitions and their values (a similar setting has also been explored in \cite{balcan2015learning,sliwinski2017learning}); \citeauthor{balkanski2017cost} show that \eqref{eq:influence-stat-cost-sharing} is a unique way of measuring player influence in the data-driven cooperative games. The only difference between \eqref{eq:influence-stat-cost-sharing} and \citeauthor{balkanski2017cost}'s measure is that rather summing over all sets $T$ such that $i \in T$, one takes only the sets in $\cal X$ that contain $i$. In other words, applying MIM to the cooperative game theoretic setting yields an influence measure that generalizes \citeauthor{balkanski2017cost}'s measure.
In classic cooperative games, one often refers to the overall influence of player $i$ in the game $v$, rather than $i$'s influence with respect to a specific set. This leads us to consider the sum-total of the setwise influence of $i$.
\begin{align}
\psi_i 	 = &\sum_{S:i \notin S}\phi_i(S) = \sum_{S: i \notin S}\sum_{T:i \in T}\frac{v(T) - v(S)}{|S\symdiff T|}\notag\\
		 = &\sum_{S:i \notin S}\phi_i(S) = \sum_{S: i \notin S}\sum_{T:i \notin T}\frac{v(T\cup\{i\}) - v(S)}{|S\symdiff T|+1}\label{eq:generalized-gt-influence}
\end{align}
One popular influence measure in cooperative games is the Banzhaf value~\cite{banzhaf}; this value equates player $i$'s influence with its average marginal contribution over all coalitions. That is, the Banzhaf value is of the form:
\begin{align}
\beta_i = \frac{1}{2^n}\sum_{S\subseteq N\setminus\{i\}}v(S\cup\{i\}) - v(S)\label{eq:banzhaf}
\end{align}
We now show that \eqref{eq:generalized-gt-influence} equals \eqref{eq:banzhaf} (up to a constant depending only on $n$).
\begin{theorem}
	The measure $\psi_i$ described in Equation \eqref{eq:generalized-gt-influence} equals $\frac{2^n(2^n - 1)}{n}\beta_i$.
\end{theorem}
\begin{proof}
	Let us write $N\setminus \{i\} = N_{-i}$. We first observe the summands in \eqref{eq:generalized-gt-influence}: note that every pair of coalitions $S,T \subseteq N_{-i}$ appears exactly twice in the summation: once with $v(S\cup\{i\})-v(T)$ in the numerator, and once with $v(T\cup\{i\}) - v(S)$ in the numerator. That is, if we let $\binom{\cal N_{-i}}{2}$ be the set of all pairs of coalitions in $N_{-i}$, then \eqref{eq:generalized-gt-influence} simply equals
	\begin{align}
	\sum_{(S,T)\in \binom{\cal N_{-i}}{2}} \frac{v(T\cup\{i\}) - v(S) + v(S\cup\{i\}) - v(T)}{|S\symdiff T|+1}&=\notag\\
	\sum_{(S,T)\in \binom{\cal N_{-i}}{2}} \frac{v(T\cup\{i\}) - v(T)+ v(S\cup\{i\}) - v(S)}{|S\symdiff T|+1}&=\notag\\	
	\sum_{S\subseteq N_{-i}} \zeta(S)(v(S\cup\{i\}) - v(S))&\label{eq:banzhafalmostthere}	
	\end{align}
	Here, $\zeta(S) = \sum_{k=0}^{n-1}\frac{\zeta_k(S)}{k+1}$,
	where $\zeta_k(S)$ is the number of coalitions $T\subseteq N_{-i}$ for which $|S\symdiff T| = k$. Let us compute the value $\zeta_k(S)$. Suppose that $|S\cap T| = r \le k$; in order to have $|S\symdiff T| = k$, $T$ must have exactly $k - r$ elements from $N_{-i}\setminus S$. In other words, $\zeta_k(S)$ equals exactly the number of ways one can choose subsets of size $r$ from $S$, and subsets of size $k - r$ from $N_{-i}\setminus S$, for all values of $r = 0,\dots,k$. Letting $|S| = m$, we have
	\begin{align}
	\zeta_k(S) =& \sum_{r = 0}^k \binom{m}{r}\binom{n-m-1}{k-r}=\binom{n-1}{k}\label{eq:zeta-chuvandermonde}
	\end{align}
	Plugging Equation \eqref{eq:zeta-chuvandermonde} into $\zeta(S)$ we obtain that $\zeta(S) = \frac{2^n - 1}{n}$.
	Putting it all together we obtain that \eqref{eq:banzhafalmostthere} simply equals $\frac{2^n(2^n-1)}{n}\beta_i$
	which concludes the proof.
\end{proof}

We mention that the measure proposed by \cite{Datta2015influence} also collapses to the Banzhaf value (times a constant dependent on $n$) in the cooperative game setting.
QII, on the other hand, can be seen as an implementation of the Shapley value to influence in a data domain. 
\section{Comparison to Existing Measures}\label{sec:existing}
In this section we provide an overview of some existing influence measures in data domains, and compare them to MIM. Measuring influence in data domains for algorithmic transparency is a relatively new approach, and has seen a veritable explosion of literature in recent years; we believe it is important to keep abreast of known methodologies and understand the domains where they are most appropriate. 
\subsection{Parzen}
The main idea behind the approach followed by \citename{baehrens2010explain} is to approximate the labeled dataset with a {\em potential function} and then use the derivative of this function to locally assign influence to features. Given a locality measure $\sigma\in \R_+$ and a kernel function 
\begin{align}
k_\sigma(\vec x) = \frac{1}{\sqrt{\pi \sigma^2}}\exp\left(\frac{- \sum_{i=1}^nx_i^2}{2\sigma^2}\right)\label{eq:parzen-kernel}
\end{align} 
The Parzen measure $\phi_{\Parzen_\sigma}(\vec x,\D)$, is given by the derivative of the potential function below at $\vec x$.
\begin{align*}
\mathbb{P}(c(\vec x) = 1 | \vec x) = \frac{\sum_{\vec y \in \D c(\vec y) = 1}k_\sigma (\vec x - \vec y)}{\sum_{\vec y \in \D}k_\sigma (\vec x - \vec y)}.
\end{align*}
It is easy to check that $\phi_{\Parzen_\sigma}$ satisfies Axioms~\ref{ax:shift} to \ref{ax:flip}. However, Parzen is neither monotonic, nor does it satisfy non-bias. To understand why Parzen fails monotonicity it helps to look at \eqref{eq:parzen-kernel}. In Figure~\ref{fig:ParzenMonotonicity}, we have a single feature ranging from $0$ to $2$; we are measuring influence for the point $\vec x_0$ (marked with a green circle). When we add two more positive labels slightly to its right, monotonicity requires that the value of $\phi_{\Parzen_\sigma}(\vec x_0,\D)$ should not decrease; however, this addition `flattens' the potential function, decreasing the influence of the feature. 
Non-bias is violated on any dataset with at least two distinct points. The underlying problem is the same: $\phi_{\Parzen_\sigma}$ measures only change in labels, so data points with the same label lead to zero influence. This leads to $\phi_{\Parzen_\sigma}$ assigning influence to random noise.

\pgfdeclareplotmark{+)}{\draw (0,0) node {$+$};}
\pgfdeclareplotmark{-)}{\draw[scale = 2](0,0) node {$-$};}
\pgfdeclareplotmark{o)}{\draw[] (0,0) circle (3pt);}
\tikzset{%
declare function={  x1 = 1.5;	x2 = 1.65; x3 = 1.8; x4=0.3;
					y1 = 0.0; y2 = 0.15; y3 = 0.45;
					sigma = 0.5;
					k(\x) = exp(-0.5*\x*\x/(sigma*sigma )/(sqrt(2*sigma*pi)) ) ; 	
					quotient(\x) = (k(\x-x1)+k(\x-y1)+k(\x-y2)+k(\x-y3)+k(\x-x4);	
					quotient2(\x) = (k(\x-x1)+k(\x-x2)+k(\x-x3)+k(\x-x4)+k(\x-y1)+k(\x-y2)+k(\x-y3));				
					myfun(\x)  = (k(\x-x1)+k(\x-x4))/(k(\x-x1)+k(\x-x4)+k(\x-y1)+k(\x-y2)+k(\x-y3));
                   	myfun2(\x)  = (k(\x-x1)+k(\x-x2)+k(\x-x3)+k(\x-x4))/(k(\x-x1)+k(\x-x2)+k(\x-x3)+k(\x-x4)+k(\x-y1)+k(\x-y2)+k(\x-y3));
                   	parzen(\x)  = (-(k(\x-y1)+k(\x-y2)+k(\x-y3))*(k(\x-x1)*(\x-x1)+k(\x-x4)*(\x-x4))+(k(\x-y1)*(\x-y1)+k(\x-y2)*(\x-y2)+k(\x-y3)*(\x-y3))*(k(\x-x1)+k(\x-x4))/(quotient(\x)*quotient(\x)*sigma*sigma);
                   	parzen2(\x)  = (-(k(\x-y1)+k(\x-y2)+k(\x-y3))*(k(\x-x1)*(\x-x1)+k(\x-x3)*(\x-x3)+k(\x-x2)*(\x-x2)+k(\x-x4)*(\x-x4))+(k(\x-y1)*(\x-y1)+k(\x-y2)*(\x-y2)+k(\x-y3)*(\x-y3))*(k(\x-x1)+k(\x-x2)+k(\x-x3)+k(\x-x4))/(quotient2(\x)*quotient2(\x)*sigma*sigma);             	                 }
}
\begin{figure}
\begin{center}
\begin{tikzpicture}[]
\begin{axis}[scale=0.5, xmin = -0.1, xmax = 2,ymin = -0.1, ymax = 1,
]
\addplot[color=gray][domain=-0.5:2] {myfun(x)};
\addplot[color = green,mark = +),only marks] coordinates  {(x1,0) (x4,0)};
\addplot[color = red,mark =-),only marks] coordinates  {(y1,0) (y2,0) (y3,0)};
\addplot[blue, quiver={u=0.25,v=parzen(x1)/8},-stealth] coordinates {(x1,myfun(x1)};
\addplot[color = green,mark =o),only marks] coordinates  {(x1,0)};
\end{axis}
\end{tikzpicture}
\begin{tikzpicture}[]
\begin{axis}[scale=0.5,xmin = -0.1, xmax = 2,ymin = -0.1, ymax = 1, yticklabels = \empty,
			legend style={
	        cells={anchor=west},
	        legend pos= north west,
			at ={(0,-0.5)},
			anchor = north west,
			legend columns =2,},legend to name=ref:leg2,
			]

\addplot[color = green,mark = +),only marks] coordinates  {(x1,0)(x2,0) (x3,0) (x4,0)};
\addlegendentry{$\vec y \in \D, c(\vec y) = 1$};

\addplot[color=black][domain=-0.5:2] {myfun2(x)};
\addlegendentry{$\mathbb{P}(c(\vec x)=1)$};
\addplot[color = green,mark =o),only marks] coordinates  {(x1,0)};
\addlegendentry{$\vec x_0$};
\addplot[blue, quiver={u=0.25,v=parzen2(x1)/8},-stealth] coordinates {(x1,myfun2(x1)};
\addlegendentry{$\phi_{\Parzen_\sigma}(\vec x_0,\D)$};
\addplot[color = red,mark =-),only marks] coordinates  {(y1,0) (y2,0) (y3,0)};
\addlegendentry{$\vec y \in \D, c(\vec y) = -1$};
\addplot[color=gray,opacity = 0.5][domain=-0.5:2] {myfun(x)};
\addplot[blue, quiver={u=0.25,v=parzen(x1)/8},-stealth,opacity=0.3] coordinates {(x1,myfun(x1)};
\legend{}
\end{axis}
\end{tikzpicture}
\end{center}
\caption{Parzen violates monotonicity; the point of interest $\vec x_0$ is marked with a green circle. Its influence is the slope of the blue arrow above it.
}
\label{fig:ParzenMonotonicity}
\end{figure}
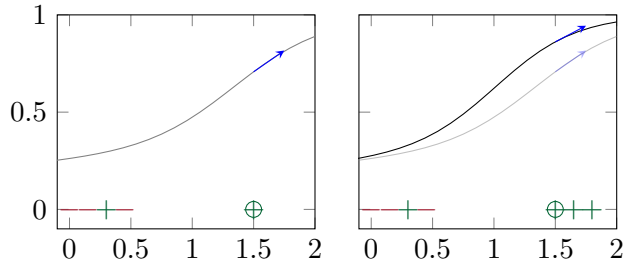
\subsection{LIME} 
\label{compareLime}
The approach followed by \citename{Ribeiro2016should} is based on the idea of using an interpretable classifier approximating the original in a region around $\vec x$; this simpler classifier then can be thought of as an explanation. This approach is termed Local Interpretable Model-agnostic Explanation (LIME).

\citeauthor{Ribeiro2016should} provide a concrete applicable framework, providing explanations in specific application domains. Some parts of this framework, however, lead to obvious violations of the axioms in Section~\ref{sec:axioms}. As an example, LIME maps datapoints to a binary explanation space, rather than considering them directly. This mapping aims to ensure that the result is human-interpretable; however, it clearly violates Axioms~\ref{ax:shift} to \ref{ax:flip}. 
On the other hand, one can draw a close connection between the theoretical framework underlying LIME, and the MIM formulation. In order to do so, it is useful to think of an influence measure as a linear classifier that approximates the data in a region close to the point of interest $\vec x$. We define the classifier based on an influence measure $\phi$ simply as $c_\phi(\vec y) = \indicator(\phi\vec y \ge \phi \vec x)$

We rewrite the core optimization problem in \citename{Ribeiro2016should}, when a linear classifier is used as an explanation:
\begin{align}
\phi_{\text{LIME}}(\vec x) = \argmin_{\phi \in \R^n} \sum_{\vec y \in \mathcal{X}} \alpha(||\vec y - \vec x||) (c(\vec y) - c_\phi (\vec y))^2 \label{eq:LIME}
\end{align}
where $\alpha$ is some non-negative function and we assume for simplicity $c(\vec x) =1$.

Comparing this to Section~\ref{sec:optimization} one can see that at its core, LIME minimizes the mean-squared error, whereas MIM maximizes cosine similarity (see Section~\ref{sec:optimization}). We note that other implementations of LIME (appearing in its source code), use cosine similarity rather than mean-squared error as the target; our results (namely Theorem~\ref{thm:charthm}) indicate that using cosine similarity offers certain theoretical guarantees over other approaches.



\subsection{Counterfactual Influence}
\citename{Datta2015influence} initiate the axiomatic analysis of influence in data domains. Unlike other measures in this section, their approach does not measure feature influence for a given point of interest; rather, it measures the {\em overall influence} of a feature for a given dataset. Following our notation, one can formulate the measure they propose as follows:
\begin{align}
\eta_i(\cal X) = \frac{1}{|\cal X|}\sum_{\vec x \in \cal X}\sum_{y_i: (\vec x_{-i},y_i) \in \cal X} |c(\vec x) - c(\vec x_{-i},y_i)|\label{eq:ijcai2015}
\end{align}
In other words, the measure proposed by \citename{Datta2015influence} does the following: when measuring the influence of the $i$-th feature; for every point $\vec x \in \cal X$, it counts the number of points in $\cal X$ which differ from $\vec x$ by only the $i$-th feature, and in their classification outcome. This follows the idea of {\em counterfactual influence}: the importance of feature $i$ is equivalent to its aggregate ability to change the outcome for points in $\cal X$, assuming that one is only allowed to change the $i$-th coordinate of $\vec x$. The axioms satisfied by \eqref{eq:ijcai2015} turn out to be too stringent: first, the counterfactual measure requires a dataset that contains datapoints differing by only one feature. Second, in many types of data, it is extremely unlikely that changing the state of a single feature will result in a change to the classification outcome (as noted by \citename{Datta2016algorithmic}); indeed, on the dataset we study (Section~\ref{sec:experiments}), Equation~\eqref{eq:ijcai2015} outputs zero influence for all features: no two points differ by only one feature.   
\subsection{Quantitative Input Influence}
\citename{Datta2016algorithmic} propose an influence measure generalizing counterfactual influence. Instead of measuring the effect of changing a single feature, they examine the {\em expected influence of changing a set of features}. More formally, given a set of features $S$, let 
$v(S;\vec x) = \E_{\vec y}[c(\vec x_{-S},\vec y_S)]$
where $(\vec x_{-S},\vec y_S)$ is the vector resulting from replacing the values of features in $S$ with those of features in $\vec y$ ($\vec y$ is sampled from the empirical distribution of $\cal X$). In other words, $v(S;\vec x)$ measures the expected effect of randomizing the values of features in $S$ on the classification outcome of $\vec x$, with samples drawn according to the empirical distribution of $S$ values in the dataset. Given this notion of `value' for a set of features, \citename{Datta2016algorithmic} use {\em the Shapley value~\cite{shapleyvalue}}, a well-known economic measure of influence from coalitional game theory. More formally, given a subset of features $S$, and a feature $i\notin S$, let 
$m_i(S;\vec x) = v(S\cup\{i\};\vec x) - v(S;\vec x)$;
 that is, $m_i(S;\vec x)$ is the {\em marginal effect} of randomizing $i$, given that we have randomized $S$. Let $\cal N_k^i = \{S\subseteq N\setminus\{i\}:|S| = k\}$; 
the influence measure defined by \citename{Datta2016algorithmic} is then
\begin{align}
\QII_i(\vec x) = \frac{1}{n!}\sum_{k = 0}^{n-1}k!(n - k - 1)!\left(\sum_{S\in \cal N_k^i}m_i(S;\vec x)\right)
\label{eq:QII}
\end{align}
$\QII_i(\vec x)$ is simply the Shapley value of feature $i$ under the coalitional game defined by $v(S;\vec x)$. By using the Shapley value, QII immediately guarantees several desirable properties `for free' (as the Shapley value satisfies them); moreover, the Shapley value (and thus, QII) is the {\em only} way of measuring influence that can satisfy these properties. However, QII suffers from two major drawbacks. The first is that when computing $v(S;\vec x)$, one assumes the ability to query the classifier on points that are not in the dataset (in particular, when computing $c(\vec x_{-S},\vec y_S)$). Secondly, computing QII is computationally intensive, both when deriving the value of a set of features in $v(S;\vec x)$ and when aggregating marginal effect in \eqref{eq:QII} (\citename{chen2018lshapley} propose workarounds to these issues).
\subsection{Black-Box Access Vs. Data-Driven Approaches}
Influence measures in data domains seem to follow either one of two paradigms. One class of methods relies on {\em black-box access to the underlying classifier}; for example, QII~\cite{Datta2016algorithmic} requires classifier queries in order to compute $v(S;\vec x)$; LIME makes such queries to sample a local region of $\vec x$. {\em Data-driven methods} (e.g. Parzen, MIM) do not require black-box access. 

Is it valid to assume black-box access to a classifier? This depends on the implementation domain one has in mind. On the one hand, having more access, measures such as QII and LIME offer better explanations in a sparse data domain; however, they are essentially unusable when one does not have access to the underlying classifier. Data-driven approaches such as MIM, the counterfactual measure and Parzen are more generic and will work on any given dataset; however, they will naturally not be particularly informative in sparse regions of the dataset. 
That said, data-driven models subsume ones assuming black-box access: any data-driven method can be used after an initial black-box query phase: in this way, we add more points to the dataset $\cal X$ as a preprocessing step (for example, in order to obtain a dense region around the point of interest), and then run the data-driven method. 

\begin{table*}
	[ht!] 
	\centering
	\begin{tabular}
		{p{0.55in}p{0.01in}p{0.5in}p{0.55in}p{0.01in}p{0.5in}p{0.55in}p{0.01in}p{0.5in}p{0.64in}p{0.1in}} \toprule
		&& \multicolumn{2}{c}{MIM}&&\multicolumn{2}{c}{Parzen}  &&\multicolumn{2}{c}{LIME}\\
		\cmidrule{3-4}
		\cmidrule{6-7}
		\cmidrule{9-10}
		POI  &&  Influence &  Shifted  && Influence&  Shifted  && Influence&  Shifted \\
		\midrule
		\parbox[c]{1em}{\includegraphics[width=0.65in]{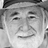}} &&
		\parbox[c]{1em}{\includegraphics[width=0.65in]{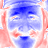}} & 
		\parbox[c]{1em}{\includegraphics[width=0.65in]{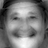}} & &
		\parbox[c]{1em}{\includegraphics[width=0.65in]{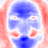}} & 
		\parbox[c]{1em}{\includegraphics[width=0.65in]{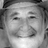}} & & 
		\parbox[c]{1em}{\includegraphics[width=0.65in]{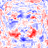}} &
		\parbox[c]{1em}{\includegraphics[width=0.65in]{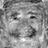}}\\ 
		\parbox[c]{1em}{\includegraphics[width=0.65in]{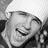}} &&
		\parbox[c]{1em}{\includegraphics[width=0.65in]{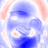}} & 
		\parbox[c]{1em}{\includegraphics[width=0.65in]{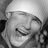}} & &
		\parbox[c]{1em}{\includegraphics[width=0.65in]{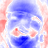}} & 
		\parbox[c]{1em}{\includegraphics[width=0.65in]{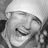}} & & 
		\parbox[c]{1em}{\includegraphics[width=0.65in]{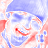}} &
		\parbox[c]{1em}{\includegraphics[width=0.65in]{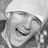}}\\  
		\bottomrule
	\end{tabular}
	\caption{Influence of two different points of interest (POI)} \label{tab:faces}
\end{table*}
\section{Experimental results}\label{sec:experiments}
In what follows, we apply MIM, Parzen and a version of LIME on a facial expression dataset. We ran our experiments using a workstation with a quad core Intel i7 CPU, and 16GB of RAM. We were able to compute each influence vector in $4-5$ seconds.
The dataset used for this experiment is a part of the Facial Expression Recognition 2013 dataset \cite{goodfeli2013contest}. The data consists of $12\,156$ $48\times 48$ pixel grayscale images of faces, evenly divided between happy and sad facial expressions. Each pixel is a feature; its brightness level is its parametric value. A parametric Parzen influence measure with $\sigma = 4.7$ and a monotone influence measure with $\alpha(d) = \frac{1}{d^2}$ were run on some of the images. Further, we used a black-box data version of LIME. For the $\alpha$ parameter in Equation~\ref{eq:LIME}, we choose $\alpha_\rho(d) =  \sqrt{\exp(-d^2/\rho^2)}$ with $\rho=3$ as a Kernel function.


The first row of Table~\ref{tab:faces} shows an example picture of a happy face from the dataset, along with a visualization of the influence vectors as produced by MIM, Parzen and LIME. In the images of influence vectors, the color blue (red) indicates positive (negative) influence; that is, for every pixel, the measures indicate that the brighter (darker) the pixel in the original image, the more `happy' (`sad') the face. The third, fifth and seventh column show the point of interest shifted according to the respective influence vector, i.e. the pixels with positive influence were brightened, and darkened if their influence was negative.
According to the MIM influence vector, the factors that contribute to this face looking happy, are a bright mouth with darkened corners, bright eyebrows, bright tone of the face, and a darkened background. Shifting the picture along the influence vector seems to make the person in the picture smile wider, and open their mouth slightly. The Parzen vector differs from the MIM vector mainly in that it suggests dark eyes as indicative of the label and does not indicate the eyebrows as strongly. LIME, while generally agreeing with the other two, results in a more 'shattered' image. Seemingly it's better for a classifier to focus it's weights on a smaller set of features, while for MIM and Parzen you can see that neighbouring pixels actually have similar influence.    

The second row shows another example picture  and its corresponding  influence vectors; however here, all measures fail to offer a meaningful explanation. This is likely to be since the face in the image is tilted, unlike the majority of images in the dataset. This is due to the fact that the dataset does not describe the locality of the image well enough; one can expect this to be the case for many images if the dataset is so small (12000) for such a complex feature space ($48\times 48 = 2304$ features, with each potentially taking $256$ different shades of gray). This exemplifies the dependency of MIM on the dataset provided, and indicates it needs a relatively dense locality in order to perform reasonably well, if black-box access to the classifier or any domain knowledge cannot be assumed.

\subsection{Additional Experiments on the Facial Expression Dataset}\label{sec:supp-facial-expression}

\newlength{\pictureWidthInTable}
\setlength{\pictureWidthInTable}{0.66in}

Tables \ref{tab:morefaces} and \ref{tab:parameters} present additional experimental results on the dataset from \cite{goodfeli2013contest}, with influence vectors computed similarly as in Section~\ref{sec:experiments} of the main paper. In Table \ref{tab:morefaces} depicted are five happy and five sad labeled images, their influence vector and the images shifted along the vector to `enhance' their label as suggested by the influence vector; we use MIM, Parzen and LIME to compute influence. MIM and Parzen produce similar vectors, while the outputs of LIME are visibly more jagged, introducing much more noise to the shifted image. As one might expect, influence vectors for opposing labels tend to have similar but inverted direction.

In Section~\ref{sec:experiments}, we test LIME and Parzen with certain parameter choices. LIME is parameterized by a variable $\rho$ governing the behavior of a distance function $\alpha_\rho$ (see the definition of LIME in Equation~\eqref{eq:LIME} of the main paper); increasing $\rho$ makes LIME assign influence to points further away from the point of interest. We also consider LIME with (inverse) cosine similarity as our choice of $\alpha$ in Equation~\eqref{eq:LIME} of the main paper; this is again parameterized by a $\mu$ parameter controlling the amount of weight placed on points further away from the point of interest. 
Parzen is parameterized by the choice of $\sigma$ in Equation~\eqref{eq:parzen-kernel} in the main paper.

Table~\ref{tab:parameters} highlights the effects of parameter choice for both LIME and Parzen. For Parzen, we vary the $\sigma$ parameter, whereas for LIME we vary $\rho$ (for the Euclidean distance method compared to in the main paper) and $\mu$ (for the cosine similarity version). Intuitively, both parameters control the locality of their respective measures. Small values imply that points closer to the point of interest are considered with more weight or, to frame it in terms of window functions, the weight of points further away is suppressed. Larger parameter values diminish this effect. 

As can be seen in Table~\ref{tab:parameters}, small parameters make the measures place a lot of weight on the point of interest itself, and the resulting influence measure is a near-replica of it. As we increase the parameters, more neighbors are considered, resulting in a more informative influence measure. 
Large values of $\rho$ and $\mu$ make LIME much noisier. This can be explained by the fact that when the parameters are sufficiently large, LIME effectively tries to fit a linear classifier to the entire dataset. This linear classifier is highly inaccurate, resulting at a rather uninformative local influence measure. Parzen doesn't suffer from this problem, it seemingly converges to a generic version of a happy face. 

Our final parameter choices aimed at striking a balance between ignoring the effect of other points in the dataset, and maintaining a locality at the point of influence; there is a rather broad range of parameters which strike this balance, so the precise parameter choice is less critical in this experiment\footnote{We tested several additional values on various images, with similar effects.}.

\begin{table*}[ht] 
\centering
  \begin{tabular}{p{0.55in}p{0.01in}p{0.5in}p{0.55in}p{0.01in}p{0.5in}p{0.55in}p{0.01in}p{0.5in}p{0.64in}p{0.1in}} \toprule
      && \multicolumn{2}{c}{MIM}&&\multicolumn{2}{c}{Parzen}  &&\multicolumn{2}{c}{LIME}\\
      \cmidrule{3-4}
      \cmidrule{6-7}
      \cmidrule{9-10}
      POI  &&  Influence & Shifted POI && Influence&  Shifted POI && Influence&  Shifted POI\\
      \midrule
\parbox[c]{1em}{\includegraphics[width=\pictureWidthInTable]{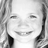}} &&
\parbox[c]{1em}{\includegraphics[width=\pictureWidthInTable]{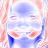}} &
\parbox[c]{1em}{\includegraphics[width=\pictureWidthInTable]{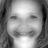}} &&
\parbox[c]{1em}{\includegraphics[width=\pictureWidthInTable]{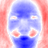}} &
\parbox[c]{1em}{\includegraphics[width=\pictureWidthInTable]{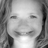}} &&
\parbox[c]{1em}{\includegraphics[width=\pictureWidthInTable]{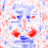}} &
\parbox[c]{1em}{\includegraphics[width=\pictureWidthInTable]{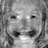}} \\
\parbox[c]{1em}{\includegraphics[width=\pictureWidthInTable]{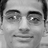}} &&
\parbox[c]{1em}{\includegraphics[width=\pictureWidthInTable]{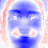}} &
\parbox[c]{1em}{\includegraphics[width=\pictureWidthInTable]{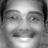}} &&
\parbox[c]{1em}{\includegraphics[width=\pictureWidthInTable]{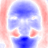}} &
\parbox[c]{1em}{\includegraphics[width=\pictureWidthInTable]{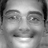}} &&
\parbox[c]{1em}{\includegraphics[width=\pictureWidthInTable]{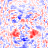}} &
\parbox[c]{1em}{\includegraphics[width=\pictureWidthInTable]{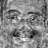}} \\
\parbox[c]{1em}{\includegraphics[width=\pictureWidthInTable]{MIM/Alpha2/268MIMPoi}} &&
\parbox[c]{1em}{\includegraphics[width=\pictureWidthInTable]{MIM/Alpha2/268MIMInf}} &
\parbox[c]{1em}{\includegraphics[width=\pictureWidthInTable]{MIM/Alpha2/268MIMOv}} &&
\parbox[c]{1em}{\includegraphics[width=\pictureWidthInTable]{Parzen/ParzenSigma4_7/Parzen268Inf}} &
\parbox[c]{1em}{\includegraphics[width=\pictureWidthInTable]{Parzen/ParzenSigma4_7/Parzen268Ov}} &&
\parbox[c]{1em}{\includegraphics[width=\pictureWidthInTable]{LIME/LinearEuclidean3/kw3_000_268Inf}} &
\parbox[c]{1em}{\includegraphics[width=\pictureWidthInTable]{LIME/LinearEuclidean3/kw3_000_268Ov}} \\
\parbox[c]{1em}{\includegraphics[width=\pictureWidthInTable]{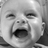}} &&
\parbox[c]{1em}{\includegraphics[width=\pictureWidthInTable]{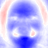}} &
\parbox[c]{1em}{\includegraphics[width=\pictureWidthInTable]{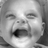}} &&
\parbox[c]{1em}{\includegraphics[width=\pictureWidthInTable]{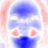}} &
\parbox[c]{1em}{\includegraphics[width=\pictureWidthInTable]{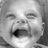}} &&
\parbox[c]{1em}{\includegraphics[width=\pictureWidthInTable]{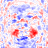}} &
\parbox[c]{1em}{\includegraphics[width=\pictureWidthInTable]{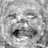}} \\
\parbox[c]{1em}{\includegraphics[width=\pictureWidthInTable]{MIM/Alpha2/490MIMPoi}} &&
\parbox[c]{1em}{\includegraphics[width=\pictureWidthInTable]{MIM/Alpha2/490MIMInf}} &
\parbox[c]{1em}{\includegraphics[width=\pictureWidthInTable]{MIM/Alpha2/490MIMOv}} &&
\parbox[c]{1em}{\includegraphics[width=\pictureWidthInTable]{Parzen/ParzenSigma4_7/Parzen490Inf}} &
\parbox[c]{1em}{\includegraphics[width=\pictureWidthInTable]{Parzen/ParzenSigma4_7/Parzen490Ov}} &&
\parbox[c]{1em}{\includegraphics[width=\pictureWidthInTable]{LIME/LinearEuclidean3/kw3_000_490Inf}} &
\parbox[c]{1em}{\includegraphics[width=\pictureWidthInTable]{LIME/LinearEuclidean3/kw3_000_490Ov}} \\
\parbox[c]{1em}{\includegraphics[width=\pictureWidthInTable]{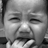}} &&
\parbox[c]{1em}{\includegraphics[width=\pictureWidthInTable]{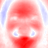}} &
\parbox[c]{1em}{\includegraphics[width=\pictureWidthInTable]{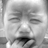}} &&
\parbox[c]{1em}{\includegraphics[width=\pictureWidthInTable]{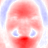}} &
\parbox[c]{1em}{\includegraphics[width=\pictureWidthInTable]{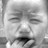}} &&
\parbox[c]{1em}{\includegraphics[width=\pictureWidthInTable]{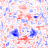}} &
\parbox[c]{1em}{\includegraphics[width=\pictureWidthInTable]{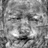}} \\
\parbox[c]{1em}{\includegraphics[width=\pictureWidthInTable]{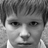}} &&
\parbox[c]{1em}{\includegraphics[width=\pictureWidthInTable]{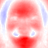}} &
\parbox[c]{1em}{\includegraphics[width=\pictureWidthInTable]{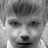}} &&
\parbox[c]{1em}{\includegraphics[width=\pictureWidthInTable]{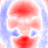}} &
\parbox[c]{1em}{\includegraphics[width=\pictureWidthInTable]{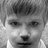}} &&
\parbox[c]{1em}{\includegraphics[width=\pictureWidthInTable]{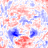}} &
\parbox[c]{1em}{\includegraphics[width=\pictureWidthInTable]{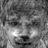}} \\
\parbox[c]{1em}{\includegraphics[width=\pictureWidthInTable]{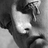}} &&
\parbox[c]{1em}{\includegraphics[width=\pictureWidthInTable]{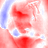}} &
\parbox[c]{1em}{\includegraphics[width=\pictureWidthInTable]{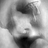}} &&
\parbox[c]{1em}{\includegraphics[width=\pictureWidthInTable]{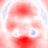}} &
\parbox[c]{1em}{\includegraphics[width=\pictureWidthInTable]{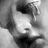}} &&
\parbox[c]{1em}{\includegraphics[width=\pictureWidthInTable]{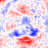}} &
\parbox[c]{1em}{\includegraphics[width=\pictureWidthInTable]{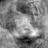}} \\
\parbox[c]{1em}{\includegraphics[width=\pictureWidthInTable]{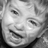}} &&
\parbox[c]{1em}{\includegraphics[width=\pictureWidthInTable]{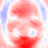}} &
\parbox[c]{1em}{\includegraphics[width=\pictureWidthInTable]{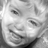}} &&
\parbox[c]{1em}{\includegraphics[width=\pictureWidthInTable]{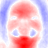}} &
\parbox[c]{1em}{\includegraphics[width=\pictureWidthInTable]{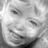}} &&
\parbox[c]{1em}{\includegraphics[width=\pictureWidthInTable]{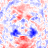}} &
\parbox[c]{1em}{\includegraphics[width=\pictureWidthInTable]{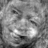}} \\
\parbox[c]{1em}{\includegraphics[width=\pictureWidthInTable]{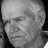}} &&
\parbox[c]{1em}{\includegraphics[width=\pictureWidthInTable]{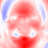}} &
\parbox[c]{1em}{\includegraphics[width=\pictureWidthInTable]{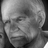}} &&
\parbox[c]{1em}{\includegraphics[width=\pictureWidthInTable]{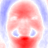}} &
\parbox[c]{1em}{\includegraphics[width=\pictureWidthInTable]{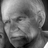}} &&
\parbox[c]{1em}{\includegraphics[width=\pictureWidthInTable]{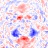}} &
\parbox[c]{1em}{\includegraphics[width=\pictureWidthInTable]{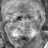}} \\ 
      \bottomrule
  \end{tabular}
\caption{Influence of ten different points of interest computed with the same parameters as in the main paper (Parzen: $\sigma = 4.7$);Lime: Euclidean distance and 3.0 kernel width). } \label{tab:morefaces}
\end{table*}

\newpage
\begin{table*}
  [ht] 
  \centering
  \begin{tabular}
      {p{0.08in} p{0.5in}p{0.55in} p{0.0001in} p{0.2in} p{0.5in}p{0.55in} p{0.00001in} p{0.08in} p{0.5in}p{0.64in}} \toprule
       \multicolumn{3}{c}{LIME (Euclidean distance)}&&\multicolumn{3}{c}{LIME (cosine similarity)}  &&\multicolumn{3}{c}{Parzen}\\
       \cmidrule{1-3}
       \cmidrule{5-7}
       \cmidrule{9-11}
      $\rho$ &  Influence & Shifted POI &&$\mu$ & Influence&  Shifted POI && $\sigma$&Influence&  Shifted POI\\
      \midrule
2.0&
\parbox[c]{1em}{\includegraphics[width=\pictureWidthInTable ]{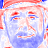}} &
\parbox[c]{1em}{\includegraphics[width=\pictureWidthInTable]{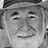}} &
&0.010&
\parbox[c]{1em}{\includegraphics[width=\pictureWidthInTable]{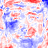}} &
\parbox[c]{1em}{\includegraphics[width=\pictureWidthInTable]{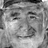}} &
&2.0&
\parbox[c]{1em}{\includegraphics[width=\pictureWidthInTable]{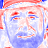}} &
\parbox[c]{1em}{\includegraphics[width=\pictureWidthInTable]{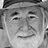}} \\
2.2&
\parbox[c]{1em}{\includegraphics[width=\pictureWidthInTable]{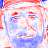}} &
\parbox[c]{1em}{\includegraphics[width=\pictureWidthInTable]{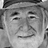}} &
&0.011&
\parbox[c]{1em}{\includegraphics[width=\pictureWidthInTable]{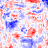}} &
\parbox[c]{1em}{\includegraphics[width=\pictureWidthInTable]{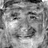}} &
&2.5&
\parbox[c]{1em}{\includegraphics[width=\pictureWidthInTable]{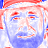}} &
\parbox[c]{1em}{\includegraphics[width=\pictureWidthInTable]{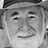}} \\
2.4&
\parbox[c]{1em}{\includegraphics[width=\pictureWidthInTable]{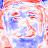}} &
\parbox[c]{1em}{\includegraphics[width=\pictureWidthInTable]{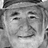}} &
&0.012&
\parbox[c]{1em}{\includegraphics[width=\pictureWidthInTable]{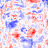}} &
\parbox[c]{1em}{\includegraphics[width=\pictureWidthInTable]{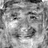}} &
&3.0&
\parbox[c]{1em}{\includegraphics[width=\pictureWidthInTable]{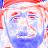}} &
\parbox[c]{1em}{\includegraphics[width=\pictureWidthInTable]{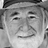}} \\
2.6&
\parbox[c]{1em}{\includegraphics[width=\pictureWidthInTable]{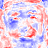}} &
\parbox[c]{1em}{\includegraphics[width=\pictureWidthInTable]{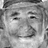}} &
&0.013&
\parbox[c]{1em}{\includegraphics[width=\pictureWidthInTable]{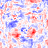}} &
\parbox[c]{1em}{\includegraphics[width=\pictureWidthInTable]{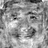}} &
&3.5&
\parbox[c]{1em}{\includegraphics[width=\pictureWidthInTable]{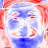}} &
\parbox[c]{1em}{\includegraphics[width=\pictureWidthInTable]{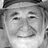}} \\
2.8&
\parbox[c]{1em}{\includegraphics[width=\pictureWidthInTable]{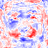}} &
\parbox[c]{1em}{\includegraphics[width=\pictureWidthInTable]{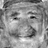}} &
&0.014&
\parbox[c]{1em}{\includegraphics[width=\pictureWidthInTable]{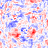}} &
\parbox[c]{1em}{\includegraphics[width=\pictureWidthInTable]{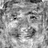}} &
&4.0&
\parbox[c]{1em}{\includegraphics[width=\pictureWidthInTable]{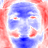}} &
\parbox[c]{1em}{\includegraphics[width=\pictureWidthInTable]{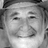}} \\
3.0&
\parbox[c]{1em}{\includegraphics[width=\pictureWidthInTable]{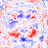}} &
\parbox[c]{1em}{\includegraphics[width=\pictureWidthInTable]{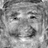}} &
&0.015&
\parbox[c]{1em}{\includegraphics[width=\pictureWidthInTable]{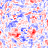}} &
\parbox[c]{1em}{\includegraphics[width=\pictureWidthInTable]{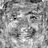}} &
&4.5&
\parbox[c]{1em}{\includegraphics[width=\pictureWidthInTable]{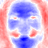}} &
\parbox[c]{1em}{\includegraphics[width=\pictureWidthInTable]{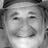}} \\
3.2&
\parbox[c]{1em}{\includegraphics[width=\pictureWidthInTable]{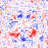}} &
\parbox[c]{1em}{\includegraphics[width=\pictureWidthInTable]{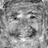}} &
&0.016&
\parbox[c]{1em}{\includegraphics[width=\pictureWidthInTable]{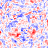}} &
\parbox[c]{1em}{\includegraphics[width=\pictureWidthInTable]{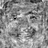}} &
&5.0&
\parbox[c]{1em}{\includegraphics[width=\pictureWidthInTable]{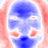}} &
\parbox[c]{1em}{\includegraphics[width=\pictureWidthInTable]{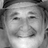}} \\
3.4&
\parbox[c]{1em}{\includegraphics[width=\pictureWidthInTable]{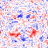}} &
\parbox[c]{1em}{\includegraphics[width=\pictureWidthInTable]{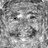}} &
&0.017&
\parbox[c]{1em}{\includegraphics[width=\pictureWidthInTable]{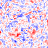}} &
\parbox[c]{1em}{\includegraphics[width=\pictureWidthInTable]{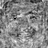}} &
&5.5&
\parbox[c]{1em}{\includegraphics[width=\pictureWidthInTable]{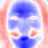}} &
\parbox[c]{1em}{\includegraphics[width=\pictureWidthInTable]{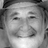}} \\
3.6&
\parbox[c]{1em}{\includegraphics[width=\pictureWidthInTable]{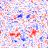}} &
\parbox[c]{1em}{\includegraphics[width=\pictureWidthInTable]{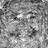}} &
&0.018&
\parbox[c]{1em}{\includegraphics[width=\pictureWidthInTable]{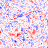}} &
\parbox[c]{1em}{\includegraphics[width=\pictureWidthInTable]{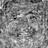}} &
&6.0&
\parbox[c]{1em}{\includegraphics[width=\pictureWidthInTable]{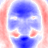}} &
\parbox[c]{1em}{\includegraphics[width=\pictureWidthInTable]{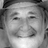}} \\
3.8&
\parbox[c]{1em}{\includegraphics[width=\pictureWidthInTable]{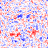}} &
\parbox[c]{1em}{\includegraphics[width=\pictureWidthInTable]{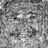}} &
&0.019&
\parbox[c]{1em}{\includegraphics[width=\pictureWidthInTable]{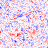}} &
\parbox[c]{1em}{\includegraphics[width=\pictureWidthInTable]{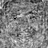}} &
&6.5&
\parbox[c]{1em}{\includegraphics[width=\pictureWidthInTable]{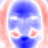}} &
\parbox[c]{1em}{\includegraphics[width=\pictureWidthInTable]{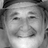}} \\
      \bottomrule
  \end{tabular}
\caption{The effect of different parameters and different distance measures.} \label{tab:parameters}
\end{table*}

\subsection{The Strategic Subject List Dataset}\label{sec:crimes}

The dataset used for this experiment is the anonymized listing of arrest data used by the Chicago Police Department's Strategic Subject Algorithm to create a risk assessment score, known as the Strategic Subject List or SSL \cite{cpd2017ssl}. The dataset contains $398,684$ records; for each record the experiment focuses on 14 attributes and a score. The scores reflect an individual's probability of being involved in a shooting incident either as a victim or an offender. Scores, ranging from 0 (low risk) to 500 (high risk), are calculated for individuals with criminal records using eight attributes (not including race or sex). These attributes are: 
age during the latest arrest (1), number of times being the victim of a shooting incident (2), number of times being the victim of aggravated battery or assault (3), number of prior arrests for violent offenses (4), presence of gang affiliation (5), number of prior narcotic arrests (6), trend in recent criminal activity (7) and number of prior unlawful use of weapon arrests (8). Additionally, in this experiment, we measure the influence of features not available to the algorithm (but which are part of the dataset): sex (9), race (10), occurrence of at least one weapon arrest in past 10 years (11), occurrence of at least one drug arrest in past 10 years (12), the subject being currently on parole (13), and the latest year of contact with the police (14).

The dataset requires some adjustments to the original definitions of MIM and Parzen. First, we adjust MIM to a regression setting by replacing the indicator $\indicator(c(\vec y) = c(\vec x))$ with $c(\vec y) - c(\vec x)$; thus, a point's contribution to the measure is weighted based on how much bigger, or smaller, its SSL score is; Parzen was similarly adjusted.
Moreover, the race feature is categorical with no natural order to its states; thus, the value of race is set to 1 whenever it is the same as the point of interest, and 0 otherwise. 

\setlength\tabcolsep{5pt} 
\begin{table}
\centering
	\caption{Two example results of MIM influence measurement. Age is given in decades, features 5, 11, 12 and 13 are binary.}
	\label{tab:crimecharts}
	\begin{tabular}{l c d{2} c d{2} }
		\toprule	
		& \multicolumn{2}{c}{Example 1}&\multicolumn{2}{c}{Example 2} \\  
		SSL value & \multicolumn{2}{c}{346}&\multicolumn{2}{c}{448} \\  
		
		\cmidrule(r){2-3} \cmidrule(lr){4-5}   
		Feature  & Value & \text{Infl.}  & Value & \text{Infl.} \\
		\midrule 
		1 Age				&	1.8	&	\clr[	66.35	]	-66.35	&	1.8	&	\clr[	52.3	]	-52.3		\\
		2 Shooting		&	0	&	\clg[	0.19	]	0.19	&	1	&	\clg[	26.97	]	26.97		\\
		3 Battery assault	&	0	&	\clg[	0.28	]	0.28	&	1	&	\clg[	26.7	]	26.7		\\
		4 Violent arrest	&	0	&	\clr[	0.21	]	-0.21	&	5	&	\clg[	100	]	133.28		\\
		5 Gang	affiliation		&	0	&	\clr[	3.21	]	-3.21	&	1	&	\clg[	22.97	]	22.97		\\
		6 Narcotics arrest&	1	&	\clg[	22.12	]	22.12	&	0	&	\clr[	7.64	]	-7.64		\\
		7 Criminal trend			&	0.3	&	\clg[	9.37	]	9.37	&	-0.6	&	\clr[	15.13	]	-15.13		\\
		8 Weapons arrest	&	0	&	\clr[	0.08	]	-0.08	&	0	&	\clr[	0.89	]	-0.89		\\
		9 Is female			&	0	&	\clr[	6.37	]	-6.37	&	0	&	\clr[	6.4	]	-6.4		\\
		10 Race			&	WWH	&	\clg[	38.55	]	38.55	&	BLK	&	\clg[	25.05	]	25.05		\\
		11 Weapon 10y		&	0	&	\clr[	0.56	]	-0.56	&	1	&	\clg[	25.98	]	25.98		\\
		12 Drug 10y		&	1	&	\clg[	21.97	]	21.97	&	0	&	\clr[	6.31	]	-6.31		\\
		13 On parole			&	0	&	\clr[	0.9	]	-0.9	&	0	&	\clr[	0.94	]	-0.94		\\
		14 Last contact	&	2016	&	\clg[	45.32	]	45.32	&	2016	&	\clg[	46.2	]	46.2		\\
		\bottomrule
	\end{tabular}

\end{table}
\setlength\tabcolsep{6pt} 
While MIM and Parzen have similar outputs on the data: the average cosine similarity between the two measures is $\ge0.94$ (taken over $\ge 8000$ randomly sampled points). Two example MIM influence measurements are depicted in Table~ \ref{tab:crimecharts}.

\citename{asher2017ssl} analyze the eight features used by the SSL Algorithm. Their results suggest that age has significant negative influence on the SSL score, while other features contribute positively in varying degrees. 
The MIM outputs confirm this statement, but suggest that the degree to which the features contribute to the SSL score vary greatly between cases. As exemplified in Table~ \ref{tab:crimecharts}, the influence of a single crime-related event tends to grow with the number of events of the same type. In the vast majority of cases, age has significant negative influence on the SSL score. Interestingly, the latest date of contact with the police often has significant positive influence on the score, despite not being used by the algorithm directly. In other cases, race has significant influence as well. This last point highlights some of the issues of using a data-driven method: features that are not used by the classifier can be assigned a significant amount of influence; this is simply because race is correlated with other variables used by the SSL algorithm. Indeed, in order to ascertain that race is not an input to the algorithm, black-box access is required.

\section{Conclusions and Future Work}
We present a novel characterization of data-driven influence measurement. Our measure is uniquely derived from a set of reasonable properties; what's more, it optimizes a natural objective function. 

Taking a broader perspective, axiomatic influence analysis in data domains is an important research direction: it allows us to rigorously discuss the {\em underlying norms} that govern our explanations. Different axioms result in alternative measures, and mathematically justifying one's choice of influence measures makes them more {\em accountable}: when explaining the behavior of classifiers in high-stakes domains, having {\em provably sound} measures offers mathematical backing to those using them. More importantly, an axiomatic approach allows one to justify the approach to non-academic stakeholders: while the proofs in this paper might be rather obscure to those without the requisite background, the axioms we use can be easily explained.

While MIM offers an interesting perspective on influence measurement, it is but a first step. First, our analysis is currently limited to binary classification domains. It is possible to naturally extend our results to regression domains, e.g. by replacing the value $\indicator(c(\vec x) = c(\vec y))$ with $c(\vec x) - c(\vec y)$; however, it is not entirely clear how one might define influence measures for multiclass domains. 

Current numerical influence measures limit their explanations to individual features; they do not capture joint effect, let alone more complex feature interactions (the only exception to this is LIME, which, at least in theory, allows fitting non-linear classifiers in the local region of the point of interest). Designing provably sound methods for measuring the effect of pairwise (or $k$-wise) interactions amongst features is a major challenge. 
Non-linear explanations naturally trade-off {\em accuracy} and {\em interpretability}. A linear explanation is easy to understand, but lacks the explanatory power of a measure that captures $k$-wise interactions. 

Finally, it is important to translate our numerical measure to an actual human-readable report. \citename{Datta2016algorithmic} propose using linear explanations as {\em transparency reports}; more advanced methods use subroutines from the classifier's source code to explain its behavior \cite{datta2017proxy,ribeiro2016programs}. Mapping numerical measures to actual human-interpretable explanations is an important open problem; we believe that analyses such as ours form the fundamental basis for making black-box systems transparent, and ultimately more accountable.
\paragraph{acknowledgments}
The authors thank the AAAI 2019 reviewers and the FAT/ML 2018 participants for their useful suggestions. The authors are supported by an NRF fellowship \#R-252-000-750-733.

\vskip 0.2in
\bibliography{literature}
\bibliographystyle{theapa}

\end{document}